\documentclass[runningheads]{llncs}
\usepackage{lmodern}
\usepackage[T1]{fontenc}
\usepackage[utf8]{inputenc}
\usepackage{color}
\usepackage{float}
\usepackage{units}
\usepackage{amsmath}
\usepackage{graphicx}
\usepackage{lmodern}

\makeatletter

\providecommand{\tabularnewline}{\\}
\floatstyle{ruled}
\newfloat{algorithm}{tbp}{loa}
\providecommand{\algorithmname}{Algorithm}
\floatname{algorithm}{\protect\algorithmname}

\usepackage{amsmath}  
\usepackage{amsfonts}
\usepackage{algpseudocode}
\usepackage[misc,geometry]{ifsym}

\makeatother

\begin{document}
\newcommand{\sidenote}[1]{\marginpar{\small \emph{\color{Medium}#1}}}

\global\long\def\model{\text{\text{\emph{cm-BO}}}}%

\title{Clustering-based Meta Bayesian Optimization with Theoretical Guarantee}
\author{Khoa Nguyen\inst{1}\textsuperscript{(\Letter)} \and Viet Huynh\inst{2} \and
Binh Tran\inst{3} \and Tri Pham\inst{3} \and Tin Huynh\inst{3} \and Thin Nguyen\inst{1}}
\authorrunning{K. Nguyen et al.}
\institute{Applied Artificial Intelligence Institute (A2I2), Deakin University, Australia 
\email{\{khoa.nguyen,thin.nguyen\}@deakin.edu.au}
\and
Edith Cowan University, Perth, Australia\\
\email{v.huynh@ecu.edu.au}
\and
The Saigon International University, Ho Chi Minh City, Vietnam
\email{\{tranlehaibinhk12,phamxuantri,huynhngoctin\}@siu.edu.vn}
}
\maketitle         
\begin{abstract}
Bayesian Optimization (BO) is a well-established method for addressing
black-box optimization problems. In many real-world scenarios, optimization
often involves multiple functions, emphasizing the importance of leveraging
data and learned functions from prior tasks to enhance efficiency
in the current task. To expedite convergence to the global optimum,
recent studies have introduced meta-learning strategies, collectively
referred to as meta-BO, to incorporate knowledge from historical tasks.
However, in practical settings, the underlying functions are often
heterogeneous, which can adversely affect optimization performance
for the current task. Additionally, when the number of historical
tasks is large, meta-BO methods face significant scalability challenges.
In this work, we propose a scalable and robust meta-BO method designed
to address key challenges in heterogeneous and large-scale meta-tasks.
Our approach (1) effectively partitions transferred meta-functions
into highly homogeneous clusters, (2) learns the geometry-based surrogate
prototype that capture the structural patterns within each cluster,
and (3) adaptively synthesizes meta-priors during the online phase
using statistical distance-based weighting policies. Experimental
results on real-world hyperparameter optimization (HPO) tasks, combined
with theoretical guarantees, demonstrate the robustness and effectiveness
of our method in overcoming these challenges.

\keywords{Bayesian optimization \and Gaussian process \and Meta
learning \and Clustering \and Optimal transport} 
\end{abstract}

\section{Introduction}

Global optimization of expensive black-box functions is a significant
challenge in scientific and industrial contexts. Bayesian Optimization
(BO) is an effective framework for this, successfully applied in hyperparameter
tuning \cite{kotthoff2017auto,snoek2012practical}, manufacturing
design \cite{huynh2023rapid,zhang2020bayesian}, and robotics \cite{driess2017constrained}.
Its popularity stems from a data-efficient sampling strategy using
stochastic surrogate models, mainly Gaussian Processes (GPs). GPs
enable efficient computation through a closed-form posterior distribution
\cite{williams2006gaussian}, providing a solid theoretical foundation
for BO guarantees. In hyperparameter tuning problems, the objective
is to identify the optimal combination of hyperparameters that maximizes
model performance on a given dataset. Often, hyperparameters have
already been tuned for multiple related datasets. By leveraging knowledge
from these previous tasks, meta-BO enables the identification of optimal
hyperparameters for a new dataset with fewer experiments compared
to using standard BO alone. To this end, several pathways exist for
exploiting insights from historical tasks to benefit new tasks such
as multitask BO \cite{swersky2013multi,dai2020multi} and meta-BO
\cite{wang2018regret,dai2022provably,wang2023hyperbo,feurer2018scalable}.
In a multi-task BO framework \cite{swersky2013multi,dai2020multi},
previous tasks are optimized concurrently with the target task using
a multi-task Gaussian Process (GP) model, which exploits similarities
between tasks. However, these approaches face significant scalability
challenges. The primary limitation arises from the need to incorporate
all data points from previous tasks and current tasks during each
optimization step. As the number of tasks increases, this results
in significantly higher computational complexity and reduced efficiency.
Recent research has explored meta-learning methodologies as a compelling
alternative \cite{wang2018regret,dai2022provably,wang2023hyperbo,feurer2018scalable}.
These approaches have introduced meta-stochastic functions that efficiently
consolidate information from prior tasks \cite{wang2018regret,feurer2018scalable,wang2023hyperbo},
providing a valuable foundation of prior knowledge for future tasks.
Furthermore, some research has focused on designing aggregated acquisition
functions tailored to enhance optimization for new tasks \cite{dai2022provably}.

A common assumption in meta-BO is that previous tasks share similarities
to the target task \cite{fan2022hyperbo+,wang2023hyperbo,wang2018regret}.
However, this assumption can be problematic, as including dissimilar
tasks may introduce noise that inhibits convergence to the optimum.
To achieve asymptotically no-regret convergence, a meta-BO approach
must selectively identify and incorporate prior knowledge from only
similar meta-tasks. To address this challenge, this paper proposes
a clustering-based meta-BO framework ($\model$), designed to select
tasks similar to the current one. 

Our method addresses key challenges in meta-learning, including computational
complexity, source-task heterogeneity, and the asymmetric setting
where target function behaviors are gradually revealed during the
online phase while offline meta-task sources remain static. To ensure
scalability, $\model$ employs a surrogate-based clustering technique
to group historical meta-tasks sharing common trends and compute a
function prototype for each cluster. For improved robustness in handling
meta-task heterogeneity, both theoretically and empirically, $\model$
combines two meta BO approaches, prior learning and ensemble model\emph{s
}\cite{bai2023transfer}, to propose an online adaptive meta-prior.
This approach introduces two key innovations: (1) a prior synthesis
procedure that adapts dynamically during the BO runtime for the target
task; and (2) an ensemble of meta-task posteriors to construct a separate
prior, with the weighting policy informed by approximated distances
between Gaussian Process (GP) posteriors. 

\textbf{Contributions.} We summarize our main contributions as follows:
(1) development of a clustering algorithm for GP posteriors from historical
meta-tasks; (2) proposal of an adaptive meta-prior synthesis strategy,
utilizing online-updated similarity-based weights during BO runtime;
(3) extensive comparison of various $\model$ variants with non-meta
and meta-BO approaches, demonstrating scalability, robustness, and
flexibility in practical BO settings; and (4) theoretical analysis
of the regret bound for clustering-based meta-BO, showing that our
meta-prior ensures computational feasibility and guarantees BO convergence,
even with high heterogeneity in historical meta-tasks.

\section{Background and Related Work }\label{sec:Related-Work}

We now review the background on Bayesian Optimization (BO) using Gaussian
Processes (GPs) and relevant statistical distances. Additionally,
we examine related studies on meta-BO, positioning our work within
this research area.

\textbf{Notation.} Denote $\mathcal{GP}(\mu,k)$ as a GP with mean
function $\mu(x):\mathbb{R}^{d}\rightarrow\mathbb{R}$ and kernel
function $k(x,x'):\mathbb{R}^{d}\times\mathbb{R}^{d'}\rightarrow\mathbb{R}$,
and $\mathcal{N}(m,\Sigma)$ as a multivariate Gaussian distribution
with mean vector $m$ and covariance matrix $\Sigma$. When $x=x'$,
the kernel function is overloaded as $k(x):=k(x,x)$. A set of $n$
observations $\{(x_{i},y_{i})\}_{i=1}^{n}$ ($x_{i}\in\mathbb{R}^{d}$,
$y_{i}\in\mathbb{R}$) is abbreviated as $\{\mathbf{X},\mathbf{y}\lvert\mathbf{X}\in\mathbb{R}^{n\times d},\mathbf{y}\in\mathbb{R}^{n}\}$,
where $\mathbf{X}:=[x_{i}^{\top}\text{]}_{i=1}^{n}$ is the set of
inputs and $\mathbf{y}:=[y_{i}]_{i=1}^{n}$ is the set of outputs.

\subsection{Background}

We first summarize standard Bayesian Optimization (BO) with Gaussian
Processes (GPs) and then introduce two function divergences---KL
divergence and Wasserstein distance---used as building blocks for
clustering the posteriors of Gaussian Processes.

\textbf{BO with GP as surrogate model. }BO can utilize a GP to update
the belief about a black-box function $f$. An important property
is that if the prior distribution follows a GP, the posterior also
follows a GP. Given $n$ observations $\mathcal{D}=\{\mathbf{X},\mathbf{y}\lvert\mathbf{X}\in\mathbb{R}^{n\times d},\mathbf{y}\in\mathbb{R}^{n}\}$,
we can estimate the posterior distribution of $f$ at an arbitrary
test point $x\in\mathbb{R}^{d}$ by computing its posterior mean and
variance function $f\mid x,\mathcal{D}\sim\mathcal{GP}(\mu_{n},k_{n})$
with:
\begin{align}
\ensuremath{\mu_{n}(x)} & =\mu_{0}(x)+k_{0}(\mathbf{X},x)^{\top}(k_{0}(\mathbf{X})+\sigma^{2}\mathbf{I})^{-1}(\mathbf{y}-\mu_{0}(\mathbf{X})),\\
k_{n}(x) & =k_{0}(x)-k_{0}(\mathbf{X},x)^{\top}(k_{0}(\mathbf{X})+\sigma^{2}\mathbf{I})^{-1}k_{0}(\mathbf{X},x),
\end{align}
where $\mu_{0}(\mathbf{X}):=[\mu_{0}(x_{i})]_{i=1}^{n}$ is a prior
mean vector, $k_{0}(\mathbf{X},x):=[k_{0}(x_{i},x)]_{i=1}^{n}$ is
a vector of co-variance between $\mathbf{X}$ and $x$, and $k_{0}(\mathbf{X}):=\left[k_{0}(\mathbf{X},x_{i})^{\top}\right]_{i=1}^{n}$
is a covariance matrix of inputs. The prior mean function $\mu_{0}$
is usually set to be zero and covariance function $k_{0}$ could be
kernel functions including Matérn, Square Exponential (SE) and Radial
basis function (RBF) kernels \cite{williams2006gaussian}. At each
BO step, the next queried point is the optimizer of an acquisition
function (AF) $\alpha$ derived from the posterior distribution, typically
Upper Confidence Bound (GP-UCB), Probability of Improvement (PI),
and Expected Improvement (EI) \cite{williams2006gaussian}. 

\textbf{Kullback-Leibler (KL) divergence.} KL divergence is a non-symmetric
measure of the information lost when the probability distribution
$Q(x)$ is used to approximate $P(x)$. In the case where $P$ and
$Q$ are two $d$-variate normal distributions with corresponding
means $m_{0}$, $m_{1}$ and (non-singular) covariance matrices $\Sigma_{0}$,
$\Sigma_{1}$, KL divergence can be analytically computed as:
\begin{equation}
D_{KL}(P\|Q)=\nicefrac{1}{2}\left(\text{tr}(\Sigma_{1}^{-1}\Sigma_{0})+\Delta m^{\top}\Sigma_{1}^{-1}\Delta m-d+\Delta\log\det\Sigma\right)
\end{equation}
where $\Delta m:=m_{1}-m_{0}$, and $\Delta\log\det\Sigma:=\log\det\Sigma_{1}-\log\det\Sigma_{0}$.
However, this divergence does not conform to the formal definition
of a metric due to the asymmetry and dissatisfying the triangle inequality.
Infinite values resulting from the unbounded nature could limit its
interpretability and pose challenges in comparing divergences across
distributions. Therefore, we adopt \emph{Jeffreys divergence} as the
symmetric version of KL distance:
\begin{equation}
D_{Jef}=D_{KL}(P\|Q)+D_{KL}(Q\|P)
\end{equation}

\textbf{Wasserstein distance and barycenter.} Derived from the \emph{optimal
transport} theory, Wasserstein $p$-distance quantifies the minimal
cost of transforming one probability distribution into another, where
the cost is defined by the $L^{p}$ distance between the distribution
masses. Let $(M,d)$ be a metric space that is a Polish space. For
$p\in[1,+\infty]$, define $\mathcal{P}_{p}(M)$ as the set of all
probability measures $\mu$ on $M$ that satisfy $\int_{M}d^{p}(x,x_{0})d\mu(x)$
is finite for some \textbf{$x_{0}\in M$}. The Wasserstein $p$-distance
between two measures $\mu,\nu\in\mathcal{P}_{p}(M)$ is given by:
\begin{equation}
W_{p}(\mu,\nu)^{p}=\inf\limits_{{\gamma\in\Gamma(\mu,\nu)}}{\displaystyle \int_{M\times M}d^{p}(x_{1},x_{2})}d\gamma(x_{1},x_{2})
\end{equation}
where $\Gamma(\mu,\nu)$ is the set of joint distributions on $M\times M$
with marginals $\mu$ and \textbf{$\nu$}. Note that \textbf{$W_{p}$}
naturally satisfies the properties of a \emph{metric} and a minimizer
from the formula of the Wasserstein $p$-distance always exists. Similar
to KL divergence, a nice closed-form formula for the \emph{$2$-Wasserstein
distance} can be achieved when two probability measures are two non-degenerate
Gaussian (i.e. normal distributions) on $\mathbb{\mathbb{R}}^{n}$.
Defining $P$, $Q$ as above, we have:
\begin{equation}
W_{2}(P,Q)^{2}=||m_{0}-m_{1}||_{2}^{2}+Tr\left(\Sigma_{0}+\Sigma_{1}-2\sqrt{\sqrt{\Sigma_{1}}\Sigma_{0}\sqrt{\Sigma_{1}}}\right)
\end{equation}

Given $n$ normal distributions $P_{1},P_{2},\dots,P_{n}$, their
Wasserstein barycenter is defined as the minimizer of $\inf_{P}\sum_{i=1}^{n}\lambda_{i}W_{p}(P_{i},P)$,
corresponding to weights $\lambda_{1},\dots,\lambda_{n}\in\mathbb{R}_{+}$.
There is a unique solution $P^{*}\sim\mathcal{N}(\bar{m},\bar{\Sigma})$
for barycentric coordinate $\{\lambda_{j}\}_{i=1}^{n}$, where $\bar{m}=\sum_{i=1}^{n}\lambda_{i}m_{i}$
and $\bar{\Sigma}$ is the unique positive definite root of $\sum_{i=1}^{n}\lambda_{i}\sqrt{\sqrt{\bar{\Sigma}}\Sigma_{i}\sqrt{\bar{\Sigma}}}=\bar{\Sigma}$,
given $P_{i}\sim\mathcal{N}(m_{i},\Sigma_{i})$. This property also
holds for a population of GPs \cite[Theorem 4 and Proposition 5]{mallasto2017learning},
when the 2-Wasserstein metric for GPs can be arbitrarily approximated
by the 2-Wasserstein metric for finite-dimensional Gaussian distributions
\cite[Theorem 8]{mallasto2017learning}, resulting in a GP barycenter.
Therefore, we can calculate the Wasserstein metric and barycenter
for GPs through their discretized Gaussian distributions.

\subsection{Related work}

\textbf{Meta-BO.} In meta-BO, observations from other meta-tasks are
leveraged to enhance the learning of representations for the prior
mean and variance functions, which can be approached in several ways.
One approach defines a joint GP kernel that combines source and target
task evaluations either by: (1) evaluating task correlations/distances
as a multi-task problem \cite{swersky2013multi,schilling2016scalable,yogatama2014efficient,tighineanu2022transfer};
or (2) accounting for noise transfer from sources to target to better
fit the data in a dual-task setting \cite{shilton2017regret}. The
complexity of these methods scales cubically with the number of data
points across all tasks. In the second approach, source tasks are
used to learn a meta-prior for the target task (prior learning). Examples
include MetaBO \cite{wang2018regret} pre-training a prior mean and
kernel function with shared input points across sources, or HyperBO
\cite{wang2023hyperbo} and HyperBO+ \cite{fan2022hyperbo+} optimizing
the marginal log-likelihood between the prior estimators and ground
truth. This group strictly assume the source and target tasks share
a common response surface, without thoroughly considering task heterogeneity
or the asymmetric setting problem\footnote{Source tasks have fixed observations, target task observations grow
during training.} between offline and online phases. The final direction uses ensemble
models, where separate posteriors or acquisition functions derived
from source tasks are combined with weighted averaging into the target
task model. This approach includes RGPE \cite{feurer2018scalable},
TAF \cite{wistuba2016two}, RM-GP-UCB \cite{dai2022provably}, RM-GP-TS
\cite{dai2022provably}, and MetaBO\_AF \cite{volpp2019meta}. The
main challenge here is to handle sparse and noisy meta-observations,
which can mislead the similarity-based weight computation.

Generally, meta-BO faces several challenges, including scalability
(especially in training-based approaches), structural heterogeneity
and sparse observations across meta-tasks, and the asymmetric setting.
However, to the best of our knowledge, there is a lack of meta-BO
methods that address all of these challenges. Combining the strengths
of all directions, our $\model$ framework bridges this gap by introducing
an online adaptive meta-prior, which is an ensemble of clustering-based
surrogate prototypes derived directly from the posteriors of source
tasks without requiring training.

\section{Clustering-based Meta Bayesian Optimization}\label{subsec:cm-BO}

By relaxing the assumption of a shared response surface among source
and target tasks in various meta-BO approaches \cite{wang2018regret,fan2022hyperbo+,wang2023hyperbo},
we facilitate \textit{meta-task heterogeneity}. This allows black-box
functions of source tasks to be drawn independently from different
GPs, varying across response surface styles. Our clustering-based
meta-BO framework consists of three main stages. We initially categorize
historical meta-tasks into homogenized groups using k-means clustering,
which is followed by the construction of generalized prototypes for
each cluster. These prototype models are then utilized for adaptive
prior learning in the target task.

\textbf{Stage 1: Meta-task clustering. }This stage corresponds to
line\textbf{~\ref{alg:start_stage_1} }to line\textbf{~\ref{alg:end_stage_1}
}of Algorithm\textbf{~\ref{alg:clus_meta_BO}}. For each meta-task
$t_{i}$, we initially train a GP prior with its meta-dataset $\mathcal{D}^{t_{i}}$,
resulting in a posterior distribution $f|x,\mathcal{D}^{t_{i}}\sim\mathcal{GP}^{t_{i}}$.
Using the GP posterior, we capture meta-task information through a
lightweight function distribution that is scalable and generalizable
over noise and sparse data, in contrast to the joint kernel design
in \cite{swersky2013multi,shilton2017regret,schilling2016scalable,yogatama2014efficient,tighineanu2022transfer}.
Subsequently, we discretize GP posteriors into Gaussian distributions
over a finite index space, each represented by a mean vector and covariance
matrix. These finite-dimensional representatives make traditional
clustering algorithms (e.g., K-means) feasible. We leverage statistical
divergences/metrics, such as Jeffrey divergence and Wasserstein metric
defined in Section\textbf{~\ref{sec:Related-Work}}, to estimate relative
distances (or dissimilarity) between two GD points. With K-means,
the center point of cluster $\mathcal{C}_{i}$ is updated as:
\begin{equation}
\bar{m}_{\mathcal{C}_{i}}:=\nicefrac{1}{|\mathcal{C}_{i}|}\sum m^{t_{j}},\text{ }\bar{\Sigma}_{\mathcal{C}_{i}}:=\nicefrac{1}{|\mathcal{C}_{i}|}\sum\Sigma^{t_{j}}
\end{equation}
where $\mathcal{GP}^{t_{j}}\in\mathcal{C}_{i}$. After clustering,
we identify $C$ \textsl{homogenized groups} of GPs.

\textbf{State 2: Cluster prototype construction. }This stage corresponds
to line\textbf{~\ref{alg:start_stage_2} }to line\textbf{~\ref{alg:end_stage_2}
}of Algorithm\textbf{~\ref{alg:clus_meta_BO}}. We generalize each
cluster of GP posteriors by an encapsulated prototype, which aims
to (1) capture the shared response surface structure of GPs within
each cluster, and (2) provide an averaged representation in cases
where there is diverse variability within a cluster. Specifically,
we design two versions of this prototype: one as a geometric center
and the other as a Wasserstein barycenter of GPs:
\begin{enumerate}
\item[(1)] \textsl{Geometric center}: The geometric center for cluster $\mathcal{C}_{i}$
can be adopted through a linear sum over GPs: $\mathcal{GP}\left(\mu^{\mathcal{C}_{i}},k^{\mathcal{C}_{i}}\right):=\sum_{\mathcal{GP}^{t_{j}}\in\mathcal{C}_{i}}\mathcal{GP}^{t_{j}}\left(\mu^{t_{j}},k^{t_{j}}\right)$,
which remains a GP with $\mu^{\mathcal{C}_{i}}=\sum_{\mathcal{GP}^{t_{j}}\in\mathcal{C}_{i}}\mu^{t_{j}}$,
and $k^{\mathcal{C}_{i}}=\sum_{\mathcal{GP}^{t_{j}}\in\mathcal{C}_{i}}k^{t_{j}}$
(Theorem~\ref{them:gd_linear}). We use its unbiased estimation $\mathcal{GP}\left(\mu^{\mathcal{C}_{i}},k^{\mathcal{C}_{i}}\right):=\mathcal{GP}\left(\nicefrac{\mu^{\mathcal{C}_{i}}}{|\mathcal{C}_{i}|},\nicefrac{k^{\mathcal{C}_{i}}}{|\mathcal{C}_{i}|}\right)$,
for empirically heterogeneous-sized clusters. 
\item[(2)] \textit{Wasserstein barycenter}: The Wasserstein barycenter concept
is derived from optimal transport, where the space of GPs is geometrically
transformed using the Wasserstein metric. For cluster $\mathcal{C}_{i}$,
we use the barycenter root with coordinates $\{\xi_{j}=\frac{1}{|\mathcal{C}_{i}|}\}_{j=1}^{|\mathcal{C}_{i}|}$.
\end{enumerate}
\textbf{Stage 3: Adaptive prior construction during the target task.}
This stage is from line\textbf{~\ref{alg:start_stage_3} }to line\textbf{~\ref{alg:end_stage_3}
}of Algorithm\textbf{~\ref{alg:clus_meta_BO}. }In the initial phase
of each query $\tau$, an online adaptive GP prior $\mathcal{GP}_{0}^{(\tau)}\left(\mu_{0}^{(\tau)},k_{0}^{(\tau)}\right)$
is synthesized from cluster prototypes:
\begin{equation}
\mathcal{GP}_{0}^{(\tau)}\left(\mu_{0}^{(\tau)},k_{0}^{(\tau)}\right):=\sum_{i=1}^{C}w_{\mathcal{C}_{i}}^{(\tau-1)}\mathcal{GP}\left(\mu^{\mathcal{C}_{i}},k^{\mathcal{C}_{i}}\right)
\end{equation}
This prior is also a GP with its mean and covariance as outlined in
Line\textbf{~\ref{alg:stage_3_1}}, Algorithm\textbf{~\ref{alg:clus_meta_BO}
}(Theorem~\ref{them:gd_linear})\textbf{.} Subsequently, a training
process using current observations on this prior yields a posterior
$\mathcal{GP}^{(\tau)}\left(\mu^{(\tau)},k^{(\tau)}\right)$. From
this posterior, we estimate the distance $d_{i}$ to each cluster
prototype $\mathcal{GP}\left(\mu^{\mathcal{C}_{i}},k^{\mathcal{C}_{i}}\right)$
using either Jeffrey divergence or Wasserstein metric. The closer
the distance $d_{i}$ is to zero, the more similar the $\tau$-th
posterior of the target task is to the majority of meta-task GPs within
cluster $\mathcal{C}_{i}$. With this intuition, we enhance the transfer
of information from clusters whose members have a high probability
of similarity in response surface structure to the target function,
and vice versa, by updating prototype weights for the next prior synthesis
(from Line\textbf{~\ref{alg:stage_3_2}} to Line\textbf{~\ref{alg:stage_3_3}},
Algorithm\textbf{~\ref{alg:clus_meta_BO}}). The weight $w_{\mathcal{C}_{i}}$
is calculated by interpolating the softmax-normalized weights that
are inversely proportional to $d_{i}$ (Line\textbf{~\ref{alg:stage_3_3}},
Algorithm\textbf{~\ref{alg:clus_meta_BO}}). By evolving the GP prior
with highly selective meta-knowledge during online queries, our framework
provides a gentle pathway to addressing meta-task heterogeneity and
the asymmetric setting. We follow standard BO experiment settings
in most literature, searching the optimizer within a finite space
$\mathcal{V}$ (i.e. a grid), making our stages 2 and 3 feasible by
using finite-dimensional GDs as approximate GPs. Constructing cluster
prototypes as geometric centers, the time complexity of $\model$
is $\mathcal{O}(Kn^{3}+T_{c}KC|\mathcal{V}|^{3}+T(n^{3}+mn^{2}d+Cm^{3}))$,
$T_{c}$, $n$, and $m$ are the number of clustering iterations,
data points, and empirical candidate points for optimizer searching.
\begin{theorem}
\label{them:gd_linear}Suppose that $\mathcal{GP}^{(1)},\mathcal{GP}^{(2)},\ldots,\mathcal{GP}^{(N)}$
are $N$ independent GPs over Euclidean space $\mathbb{R}^{d}$. Their
linear combination $\hat{\mathcal{GP}}:=\sum_{i=1}^{N}a_{i}\mathcal{GP}^{(i)}$
(where $a_{i}\in\mathbb{R}_{+}$) is also a GP over $\mathbb{R}^{d}$
\cite{adler1990introduction}. If $\mathcal{GP}^{(i)}$ has mean $\mu^{(i)}$
and covariance $k^{(i)}$ (for $i=\overline{1,N}$), then $\hat{\mathcal{GP}}$
has mean $\hat{\mu}=\sum_{i=1}^{N}a_{i}\mu^{(i)}$ and covariance
$\hat{k}=\sum a_{i}^{2}k^{(i)}$.
\end{theorem}

\begin{proposition}
\label{prop:prior_non_degenerate}If GP posteriors are derived from
meta-tasks by training a GP prior with a positive-definite kernel
function (e.g., the Matérn kernel), and cluster prototypes are geometric
centers, then the GP meta-priors constructed during the target task
are non-degenerate.
\end{proposition}

\begin{remark}
The non-degeneracy of our meta-priors guarantees the \emph{invertibility}
required for computing predictive posterior distributions in any index
subspace. In cases where cluster prototypes are \emph{Wasserstein
barycenter}s, it has been noted in \cite{mallasto2017learning,masarotto2019procrustes}
that the non-degeneracy of the barycenter of non-degenerate GPs remains
a conjecture. However, with discretized finite-dimensional GDs, there
exists a unique positive-definite covariance matrix that satisfies
the defined equation for the covariance matrix of the\emph{ barycenter}
with the given barycentric coordinates \cite{mallasto2017learning}.
Therefore, it is empirically feasible to search for the optimum of
the black-box function within a finite index space.

We now proceed to state Theorem~\ref{them:regret}, which provide
the regret bounds of our proposed method when employing the \textit{geometric
center} in cluster prototype construction and utilizing the GP-UCB
acquisition function. For \textit{Wasserstein barycenters} and other
types of AFs, we leave this as future work. 
\end{remark}

Similar to \cite[Lemma 2, 3, 6]{dai2022provably}, we define $\tilde{\mu}^{t_{i}}$
and $\tilde{k}^{t_{i}}$ as the mean and covariance of the posterior
$\tilde{\mathcal{GP}}^{t_{i}}$, which is derived by conditioning
the prior on $\tilde{\mathcal{D}}^{t_{i}}$, instead of conditioning
on $\mathcal{D}^{t_{i}}$ as Line\textbf{~\ref{alg:stage_1_1}}, Algorithm\textbf{~\ref{alg:clus_meta_BO}}.
$\tilde{\mathcal{D}}^{t_{i}}$ contains meta-observations obtained
by hypothetically observing the true function $f$ exactly at input
locations $\mathbf{X}^{t_{i}}$ observed in $\mathcal{D}^{t_{i}}$.
Let $n_{i}$ denote the number of observations in $\mathcal{D}^{t_{i}}$,
$\sigma$ represent the observation noise, and $D_{i}$ be defined
as the function gap $D_{i}:=\max_{j=1,\ldots,n_{i}}\lvert f(x_{i,j})-f_{i}(x_{i,j})\lvert$,
where $f$ and $f_{i}$ (for $i=\overline{1,K}$) are true functions
of the target task and $K$ meta-tasks. 
\begin{algorithm}[H]
\centering
\begin{algorithmic}[1]
\Require{$K$ source tasks $\mathcal{D}^{t_{i}}=\{\mathbf{X}^{t_{i}},\mathbf{y}^{t_{i}}\lvert\mathbf{X}^{t_{i}}\in\mathbb{R}^{n_{i}\times d},\mathbf{y}\in\mathbb{R}^{n_{i}}\}$ ($i=\overline{1,K}$), $N$ initial points of the target task $\mathcal{D}=\{(x_{i},y_{i})\}_{i=1}^{N}$, acquisition function (AF) $\alpha$.}
\Ensure{Estimated maximizer of the true target function $f$ after $T$ BO queries.}
\State \textbf{Stage 1.} Meta-task clustering \label{alg:start_stage_1}
\State $GP\_posteriors \gets \emptyset$
\For{$i=1,\ldots,K$}
	\State $\mathcal{GP}^{t_i}(\mu^{t_i}, k^{t_i}) \gets fit\_GP(\mathcal{GP}_{prior}, \mathcal{D}^{t_{i}})$ \Comment{train GP prior with $\mathcal{D}^{t_{i}}$} \label{alg:stage_1_1}
	\State $GP\_posteriors \gets GP\_posteriors \cup \{\mathcal{GP}^{t_i}(\mu^{t_i}, k^{t_i})\}$
\EndFor
\State $\{\mathcal{C}_{i}\}_{i=1}^{C} \gets CLUSTERING(GP\_posteriors)$  \label{alg:end_stage_1}
\State \textbf{Stage 2.} Cluster prototype construction \label{alg:start_stage_2}
\State $\mathcal{GP}(\mu^{\mathcal{C}_i}, k^{\mathcal{C}_i}) \gets get\_cluster\_prototype(\mathcal{C}_{i})$, \textbf{for} $i=1,\ldots,C$  
\label{alg:end_stage_2}
\State \textbf{Stage 3.} Adaptive prior construction during the target task \label{alg:start_stage_3}
\State $w_{\mathcal{C}_{i}}^{(0)} \gets \nicefrac{1}{C}$ \Comment{initialize prototype weights}
\For{$\tau=1,\ldots,T$}
	\State $\mu_{0}^{(\tau)}\gets\sum_{i=1}^{C}w_{\mathcal{C}_{i}}^{(\tau-1)}\mu^{\mathcal{C}_{i}}, k_{0}^{(\tau)}\gets\sum_{i=1}^{C}\left(w_{\mathcal{C}_{i}}^{(\tau-1)}\right)^{2}k^{\mathcal{C}_{i}}$ \Comment{construct GP prior} \label{alg:stage_3_1}
	\State $\mathcal{GP}^{(\tau)}(\mu^{(\tau)}, k^{(\tau)}) \gets fit\_GP(\mathcal{GP}_{0}^{(\tau)}(\mu_{0}^{(\tau)},k_{0}^{(\tau)}), \mathcal{D})$ \Comment{train adaptive GP prior}
	\State $x_{\tau} \gets \text{argmax}  \alpha(\mu^{(\tau)}, k^{(\tau)})$ \Comment{get estimated maximizer from AF}
	\State $\mathcal{D} \gets \mathcal{D} \cup \{(x_{\tau},f(x_{\tau}))\}$
	\State $d_i \gets calculate\_dist\left(\mathcal{GP}(\mu^{\mathcal{C}_i}, k^{\mathcal{C}_i}),\mathcal{GP}^{(\tau)}(\mu^{(\tau)}, k^{(\tau)})\right)$, \textbf{for} $i=1,\ldots,C$\label{alg:stage_3_2}
	\State $d_{max} \gets \max\limits_{i=1,\ldots,C}d_{i}, s \gets \sum_{i=1}^{C}e^{1-\nicefrac{d_{i}}{d_{max}}}$
	\State $w_{\mathcal{C}_{i}}^{(\tau)}\gets\nicefrac{e^{1-\nicefrac{d_{i}}{d_{max}}}}{s}$, \textbf{for} $i=1,\ldots,C$ \Comment{update prototype weights} \label{alg:stage_3_3}
\EndFor \label{alg:end_stage_3}
\State \Return $\text{max} \{x_i \lvert (x_i,\cdot)\in\mathcal{D}\}$

\end{algorithmic}

\caption{Clustering-based Meta Bayesian Optimization}\label{alg:clus_meta_BO}
\end{algorithm}

\begin{theorem}
\label{them:regret}\textbf{\textup{(Regret bound)}} Define $r_{\tau}=\max f(x)-f(x_{\tau})$
as the\emph{ instantaneous regret}, where \textup{$x_{\tau}$} is
the observation queried at iteration $\tau$. Let $x^{*}$ denote
the global maximizer of black-box function $f$. At query $\tau$,
we define the GP-UCB acquisition function as $a^{(\tau)}(x)=\mu^{(\tau)}(x)+\xi\sqrt{k^{(\tau)}(x)}$
($\xi>0$), which is constructed from mean $\mu^{(\tau)}$ and covariance
\emph{$k^{(\tau)}$ }of the GP posterior $\mathcal{GP}^{(\tau)}$.
Let $\mathbf{X_{\tau}}$ denote available observations used to train
prior. Let $\gamma_{\tau}$ represent the maximum information gain
about $f$ after querying $\tau$ observations, and define $\beta_{\tau}=B+\sigma\sqrt{2(\gamma_{\tau-1}+1+log(\nicefrac{4}{\delta}))}$
\cite[Theorem 1]{dai2022provably}\emph{.} $r_{\tau}$ is upper-bounded
by:
\begin{equation}
r_{\tau}\le\delta_{\tau}\left(\iota_{\tau}\alpha_{\tau}+\eta^{\tau}(x^{*})-\eta^{\tau}(x_{\tau})+\omega_{\tau}\left\Vert \eta^{\tau}(\mathbf{X_{\tau}})\right\Vert _{2}\right)+\nicefrac{\xi}{\beta_{\tau}}(\xi+\beta_{\tau})k_{\tau}^{\nicefrac{1}{2}}(x_{\tau})
\end{equation}
where $\eta^{\tau}(x):=\sum_{i=1}^{K}\zeta_{i}^{\tau}\left(\tilde{\mu}^{t_{i}}(x)-f(x)\right)$,
$\omega_{\tau}:=2\sum_{i=1}^{K}\zeta'{}_{i}^{\tau}\nicefrac{\sqrt{n_{i}}}{\sigma^{2}}$,
$\iota_{\tau}:=2+\omega_{\tau}\sqrt{\tau-1}$, $\delta_{\tau}:=\nicefrac{\xi}{\beta_{\tau}}\left(1-\nicefrac{\beta_{\tau}}{\xi}\right)$,
and $\alpha_{\tau}:=\sum_{i=1}^{K}\zeta_{i}^{\tau}\nicefrac{n_{i}}{\sigma^{2}}\left(D_{i}+2\sqrt{2\sigma^{2}log\nicefrac{8n_{i}}{\delta}}\right)$
with $\zeta_{i}^{\tau}:=\nicefrac{w_{\mathcal{C}_{i}}^{(\tau-1)}}{|\mathcal{C}_{i}|}$
and $\zeta'{}_{i}^{\tau}:=\nicefrac{\left(w_{\mathcal{C}_{i}}^{(\tau-1)}\right)^{2}}{|\mathcal{C}_{i}|}$\emph{. }
\end{theorem}

\begin{remark}
\label{rem:regret}The term $\nicefrac{\xi}{\beta_{\tau}}(\xi+\beta_{\tau})k_{\tau}^{\nicefrac{1}{2}}(x_{\tau})$
is proportional to the upper bound on the instantaneous regret for
standard GP-UCB \cite{srinivas2009gaussian}. Therefore, meta-tasks
affect the upper bound on the regret of our proposed method in Theorem\textbf{~\ref{them:regret}}
through the remaining term. Analytically, $\alpha_{\tau}$ contains
the linear sum of $n_{i}D_{i}$, where $D_{i}$ is the function gap
between the target and meta-task, and the linear sum of $n_{i}\sqrt{\log n_{i}}$,
weighted by $\zeta_{i}^{\tau}:=\nicefrac{w_{\mathcal{C}_{i}}^{(\tau-1)}}{|\mathcal{C}_{i}|}$.
Since \emph{$\sum_{i=1}^{K}\zeta_{i}^{\tau}=1$}, this means $\alpha_{\tau}$
can become progressively tighter during sequential BO queries; that
is, dissimilar meta-task clusters (large $D_{i}$) approach smaller
weights with less knowledge transferred. Similarly, $\eta^{\tau}(x)$
is the linear sum of $\left(\tilde{\mu}^{t_{i}}(x)-f(x)\right)$,
weighted by $\zeta_{i}^{\tau}$. When $n_{i}\rightarrow\infty$, $\tilde{\mu}^{t_{i}}(x)$
asymptotically approaches $f(x)$ and $\eta^{\tau}(x)\rightarrow0$.
For small $n_{i}$, i.e., sparse meta-task observations, the corresponding
weight $\zeta_{i}^{\tau}$ of meta-task $t_{i}$ decreases, reflecting
its diminishing impact on the bound due to dissimilarity. This also
illustrates the efficacy of our method in addressing data uncertainty.
\end{remark}

\section{Experiments}\label{sec:Experiments}

\subsection{Experimental Settings}

\textbf{Datasets.}\emph{ HPO-B} \cite{arango2021hpo}, a large-scale
benchmark for HPO based on OpenML, includes various search spaces
for machine learning (ML) tasks such as SVM, decision trees, random
forests, and XGBoost. We select 4 search spaces from its\emph{ HPO-B-v2}
subset\footnote{https://github.com/releaunifreiburg/HPO-B}, each
containing at least 10 meta-tasks and 100 evaluations per task, spanning
4 different ML models for our experiments. The data dimensions of
these search spaces range from 3 to 10 (Table \ref{tab:dataset_stats}).
The work in \cite{wang2023hyperbo} discussed the \emph{negative transfer
effect} in this benchmark, attributed to noisy meta-tasks with heterogeneous
response surfaces among meta-tasks. Through evaluation, we demonstrate
the robustness, generalizability, and scalability of $\model$ in
mitigating these challenges.

\begin{table}[H]
\centering
\vspace{-0.15cm}

\begin{minipage}[c]{0.45\columnwidth}%
\caption{ Four selected search spaces from \emph{HPO-B-v2} \cite{arango2021hpo}.
\#HPs: Number of hyperparameters (search space dimensionality); \#Tasks:
Total number of tasks.}\label{tab:dataset_stats}
\end{minipage}%
\begin{minipage}[c]{0.05\columnwidth}%
\textcolor{white}{blank space}%
\end{minipage}%
\begin{minipage}[c]{0.45\columnwidth}%
\begin{tabular}{ccc}
\hline 
\textbf{ML model} & \textbf{\#HPs} & \textbf{\#Tasks}\tabularnewline
\hline 
rpart.preproc & $3$ & $44$\tabularnewline
rpart & $6$ & $69$\tabularnewline
svm & $8$ & $63$\tabularnewline
ranger & $10$ & $74$\tabularnewline
\hline 
\end{tabular}%
\end{minipage}

\vspace{-0.15cm}
\end{table}

\textbf{Baseline methods. }\emph{First}, we compare two groups of
BO algorithms: \emph{non-meta}-\emph{BO} and \emph{meta-BO}. The former
includes (1) random search and (2) standard GP, which trains only
on observations from the target task, while the latter encompasses
three state-of-the-art approaches: (1) RGPE \cite{feurer2018scalable}
with ensemble surrogate models and mis-ranked pair-based weighting,
(2) RM-GP-UCB \cite{dai2022provably} with meta-task transfer through
GP-UCB acquisition function (AF), and (3) HyperBO \cite{wang2023hyperbo}
with prior-learning via a pre-trained surrogate prior.\emph{ Second},
we evaluate five different variants of our $\model$ framework: (1)
JefClus\_JefCMP uses Jeffreys divergence for both clustering and weight
computation, (2) WssClus\_WssCMP uses Wasserstein metric for both
tasks, (3) JefClus\_WssCMP uses Jeffreys for clustering and Wasserstein
for weights, (4) WssClus\_JefCMP is the reverse of (3)---all four
use \emph{geometric center} as the cluster prototype---and (5) WssClus\_WssCMP\_Bary
is similar to (2) but uses \textit{Wasserstein} \emph{barycenter}
instead of \emph{geometric center}.\emph{ Third}, we evaluate $\model$
variants without the clustering stage: (1) GlobalCen uses a global
geometric center of all meta-task posteriors as a fixed meta-prior
for the target task, and (2) IndiWeight-Jef, which aggregates the
geometric center with weighted meta-task posteriors based on online
similarity (we avoid the Wasserstein metric version due to its high
time cost). Our aim in comparing various versions of $\model$ is
to demonstrate its ability to balance computational efficiency in
querying observations with the overall convergence rate of meta BO
tasks. Finally, we evaluate each method using three types of AF: GP-UCB,
PI, and EI, except for RM \cite{dai2022provably}, which only implements
GP-UCB.

\textbf{Evaluation protocols and parameter settings.} We evaluate
all algorithms in a $T$-iteration BO setting, where for each query,
we extract the corresponding best-ever observations to calculate the
\emph{normalized simple regret \cite{fan2022hyperbo+}} (NSR). We
randomly split each search space into train and test sets at an $85:15$
ratio, repeating the split $5$ times to diversify the information
sources for meta-learning and test robustness of methods. The training
set contains source tasks for meta-learning. Meanwhile, each task
in the test set is used as the target task, initialized with $5$
random observations per BO run, repeated $8$ times for each pair
of training set and target task. This leads to approximately $900$
runs per method, with each BO run consisting of $50$ queries. Each
meta-task in train sets contains $50$ random historical observations
uniformly selected. The hyperparameters for HyperBO \cite{wang2023hyperbo}
and RM-GP-UCB \cite{dai2022provably} are set according to the published
works. We implement RGPE using \emph{botorch}\footnote{https://botorch.org/}
and tune its hyperparameters to achieve best validation results. For
our method, we tune the number of\emph{ }\textit{\emph{clusters $C$}}
between $2$ and $6$ to achieve the best performance across search
space splits. We evaluate clustering quality to determine the optimal
number of clusters using two metrics: 
\begin{enumerate}
\item[(1)] \textit{Intra-cluster entropy:}
\begin{equation}
intraCE=\nicefrac{\left(\sum_{i=1}^{C}\left(\nicefrac{2\sum_{t_{j},t_{k}\in\mathcal{C}_{i}}wssdist(\text{\ensuremath{\mathcal{GP}}}^{t_{j}},\text{\ensuremath{\mathcal{GP}}}^{t_{k}})}{|\mathcal{C}_{i}|(|\mathcal{C}_{i}|-1)}\right)\right)}{C}
\end{equation}
\item[(2)] \textit{Inter-cluster separation:}
\begin{equation}
interCS=\nicefrac{\sum_{\substack{t_{k}\in\mathcal{C}_{i},t_{l}\in\mathcal{C}_{j},i\neq j\\
\\}
}wssdist(\text{\ensuremath{\mathcal{GP}}}^{t_{k}},\text{\ensuremath{\mathcal{GP}}}^{t_{l}})}{\sum_{i\neq j}|\mathcal{C}_{i}||\mathcal{C}_{j}|}
\end{equation}
\end{enumerate}
where $wssdist$ is the Wasserstein distance. The larger $interCS$
and the smaller $intraCE$, the better the clustering. For GP discretization,
the number of observed locations for the \textit{\emph{finite index
space}} is set to $100$. We use a Matérn kernel with smoothness $\nu=\frac{3}{2}$
for adaptive prior construction, while an element-wise product kernel
composed of Matérn kernels with smoothness $\nu=\frac{3}{2}$ and
$\nu=\frac{1}{2}$ is used to learn posteriors for clustering. This
balance enables effective modeling of function structures with limited
observation data while maintaining robustness against noise. Regarding
AFs, following settings in \cite{wang2023hyperbo}, explore-exploit
parameter $\beta$ is set to $3$ for GP-UCB, and the max-value version
of PI is utilized.

\subsection{Experimental Discussion}

Fig. \ref{fig:regret-4796} presents experimental results for\emph{
rpart.preproc} search space, showing the \emph{mean normalized simple
regret (MNSR)} for $3$ types of AF across methods, excluding RM-GP-UCB.
We conduct two additional experiments: (1) report the average rank
of each method (aggregated across all BO runs), (2) compare performance
profiles by calculating the fraction of BO runs achieving a regret
lower than the threshold $C=0.005$. This is shown in Fig. \ref{fig:performance-4796}(a)
for \emph{rpart.preproc} search space and in Fig. \ref{fig:performance-4796}(b)
for aggregated results of remaining three search spaces.
\begin{figure}[H]
\centering
\begin{minipage}[c]{0.2\columnwidth}%
\caption{MNSR across all compared methods, on 3 types of AF from left to right:
GP-UCB, EI, PI (experimental results\emph{ on rpart.preproc }search
space)}\label{fig:regret-4796}
\end{minipage}%
\begin{minipage}[t]{0.02\columnwidth}%
\textcolor{white}{blank}%
\end{minipage}%
\begin{minipage}[c]{0.78\columnwidth}%
\includegraphics[width=1\linewidth,totalheight=4.8cm]{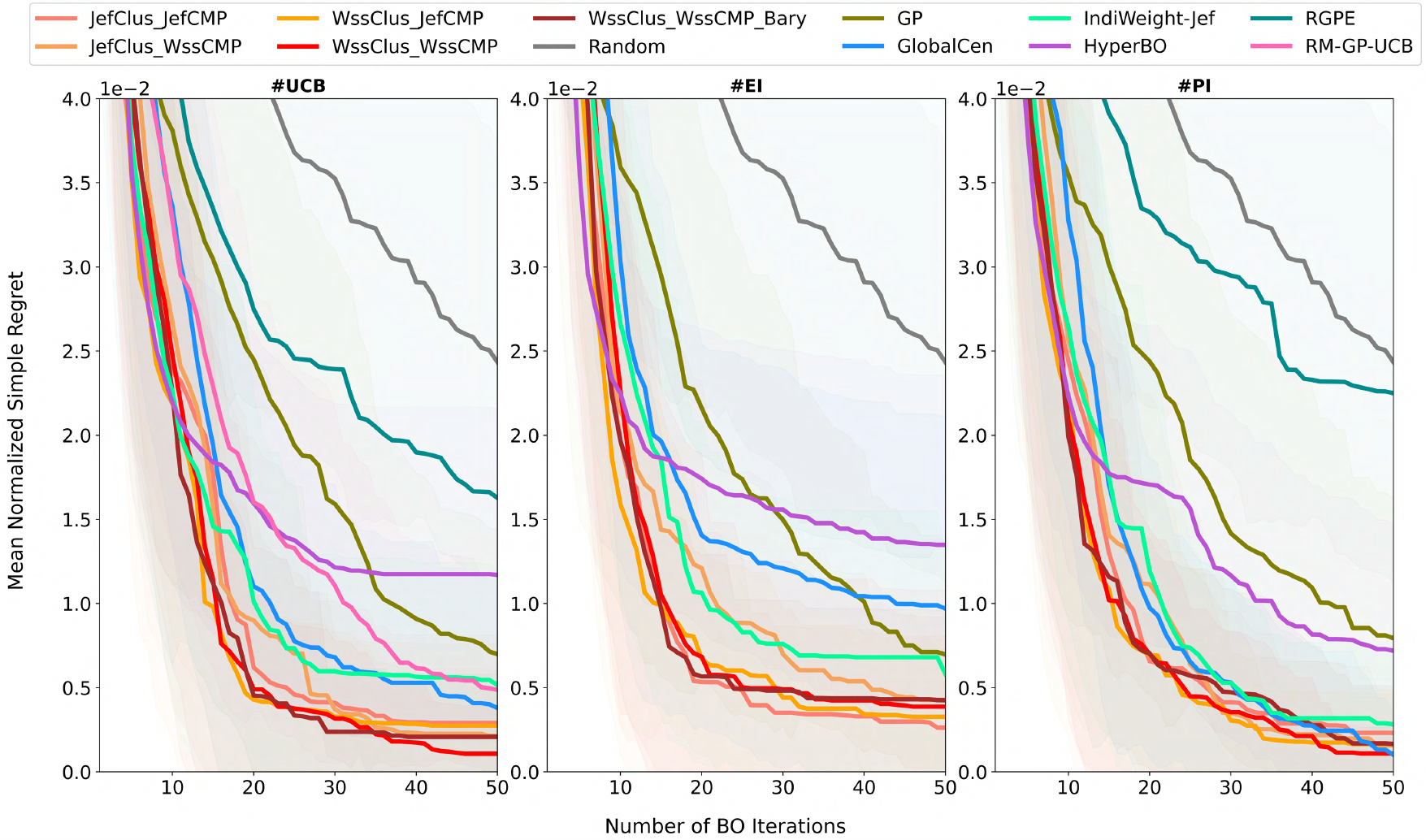}%
\end{minipage}
\end{figure}

The $\model$ variants outperform other compared methods across most
experimental results. Random Search and standard GP are less effective
as they do not leverage historical meta-task information. While RGPE
integrates meta-knowledge, it converges slowly on average, possibly
due to its ranking-based weighting policy being sensitive to noise
and heterogeneous meta-tasks. RM-GP-UCB is only feasible on the GP-UCB
AF, though it is generally the best in this setting among other methods.
HyperBO achieves fast convergence initially, but plateaus from the
$20$-th iteration and has fewer solvable BO runs at a very small
threshold of $C=0.005$, likely due to overfitting of the prior training.
In contrast, $\model$ ensures stable convergence across the entire
BO run, effectively incorporating prior knowledge to reduce early
exploration while maintaining a stable degree of exploration later.
Additionally, when the current task accumulates sufficient information,
the later stages of the BO runs can also witness a high-focused exploitation,
which accelerates convergence.

Among $\model$ variants, WssClus\_WssCMP and WssClus\_WssCMP\_Bary
exhibited the best convergence performance. This is reasonable because
the Wasserstein metric is more computationally stable, though it has
a higher time complexity as a trade-off. WssClus\_WssCMP outperformed
all other methods in terms of the number of solvable BO tasks and
average ranking, especially in the later stages of the runs. Two non-cluster-integrated
baselines, GlobalCen (non-weighted) and IndiWeight-Jef (task-specific
weighted), were less effective than the cluster-integrated $\model$,
demonstrating the efficacy of clustering in (1) prioritizing potentially
similar meta-tasks and (2) offering a suitable degree of exploration
in early BO iterations by smoothing intra-cluster meta-tasks. 

Notably, when evaluating on high-dimensional meta-datasets in HPO-P
(\emph{rpart}, \emph{svm}, and \emph{ranger} with search space dimensions
of 6, 8, and 10 respectively), the results visualized in Fig. \ref{fig:performance-4796}(b)
show that WssClus\_WssCMP\_Bary surpasses WssClus\_WssCMP, becoming
the best performer in general, although both methods maintain top
performance. In contrast, the Jeffreys-based $\model$ variants appear
to be less effective on these three high-dimensional meta-datasets,
possibly due to computational instability.

The training time for a HyperBO prior is substantial, approximately
$200$-$215$ \emph{min}, excluding later BO runtime (on the \textit{rpart.preproc}
search space). Our proposed $\model$\emph{ }\textit{\emph{framework}}
has much more reasonable timing---K-means clustering on $100$-dimensional
posteriors takes only $\approx2-3$ \emph{sec} with Jeffreys and $\approx7$
\emph{sec} with Wasserstein per clustering iteration, with typically
less than $10$ iterations. For online adaptive prior construction,
the average time to estimate the statistical distance on a $300$-location
grid (corresponding to $300$-dim GDs) is $0.02$-$0.05$ \emph{sec}
with Jeffreys and $0.4$-$0.5$ \emph{sec} with Wasserstein, scaling
slightly with the small number of clusters. Although $\model$ calculates
weights and updates the prior at each query, Fig. \ref{fig:regret-4796}
and \ref{fig:performance-4796} show its variants surpassing HyperBO
at relatively early BO iterations when sufficient information about
the target function is gained, highlighting the superior efficiency
of $\model$.
\begin{figure}[H]
\caption{Average ranking (left) and fraction of solvable BO runs at $C=0.005$
(right)}\label{fig:performance-4796}
\vspace{-0.15cm}%
\begin{minipage}[t][1\totalheight][c]{0.5\columnwidth}%
{\scriptsize (a) rpart.preproc search space}{\scriptsize\par}

\includegraphics[width=1\linewidth]{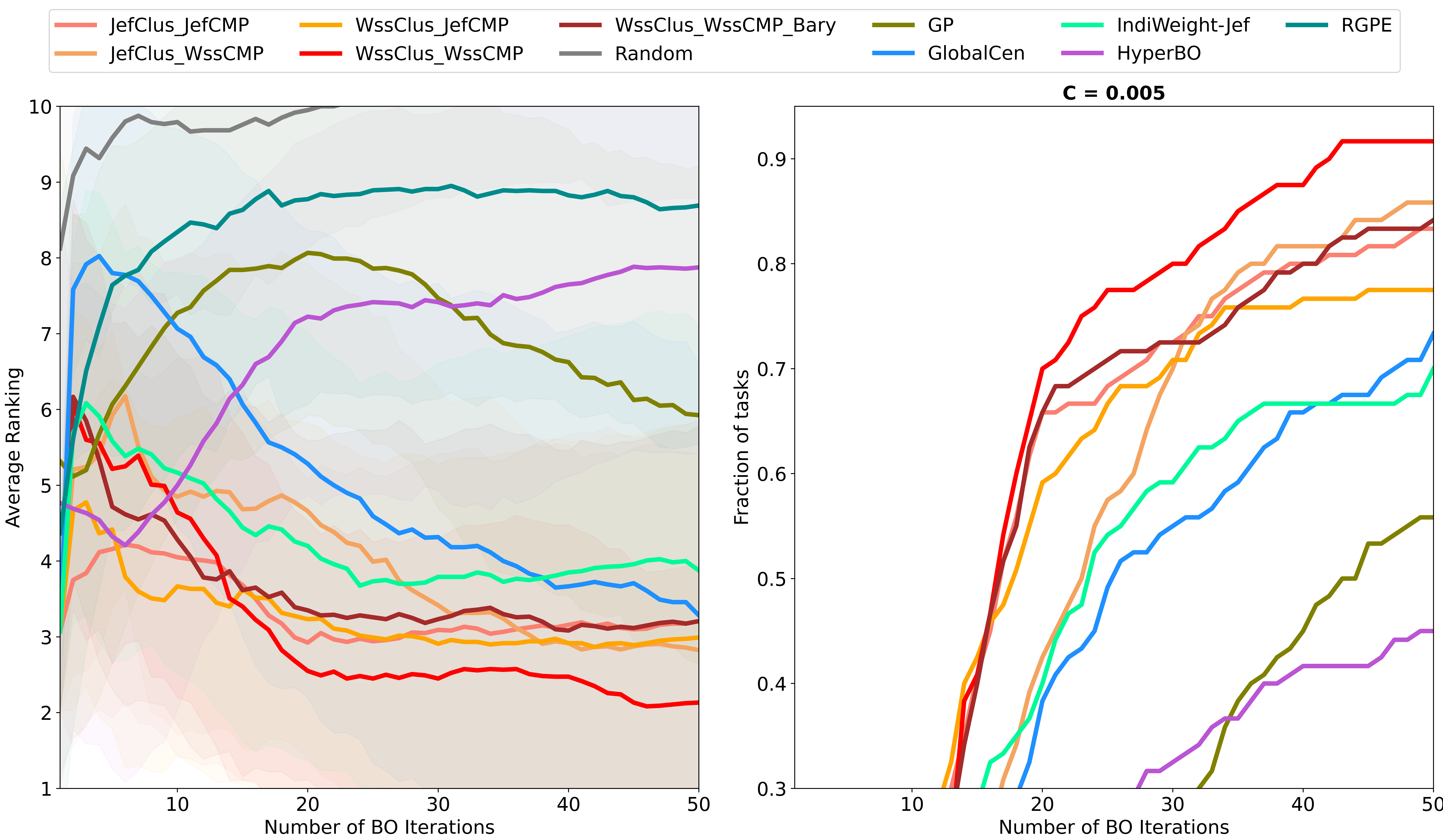}%
\end{minipage}%
\begin{minipage}[t][1\totalheight][c]{0.5\columnwidth}%
{\scriptsize (b) rpart, svm, ranger search spaces (merged)}{\scriptsize\par}

\includegraphics[width=1\linewidth]{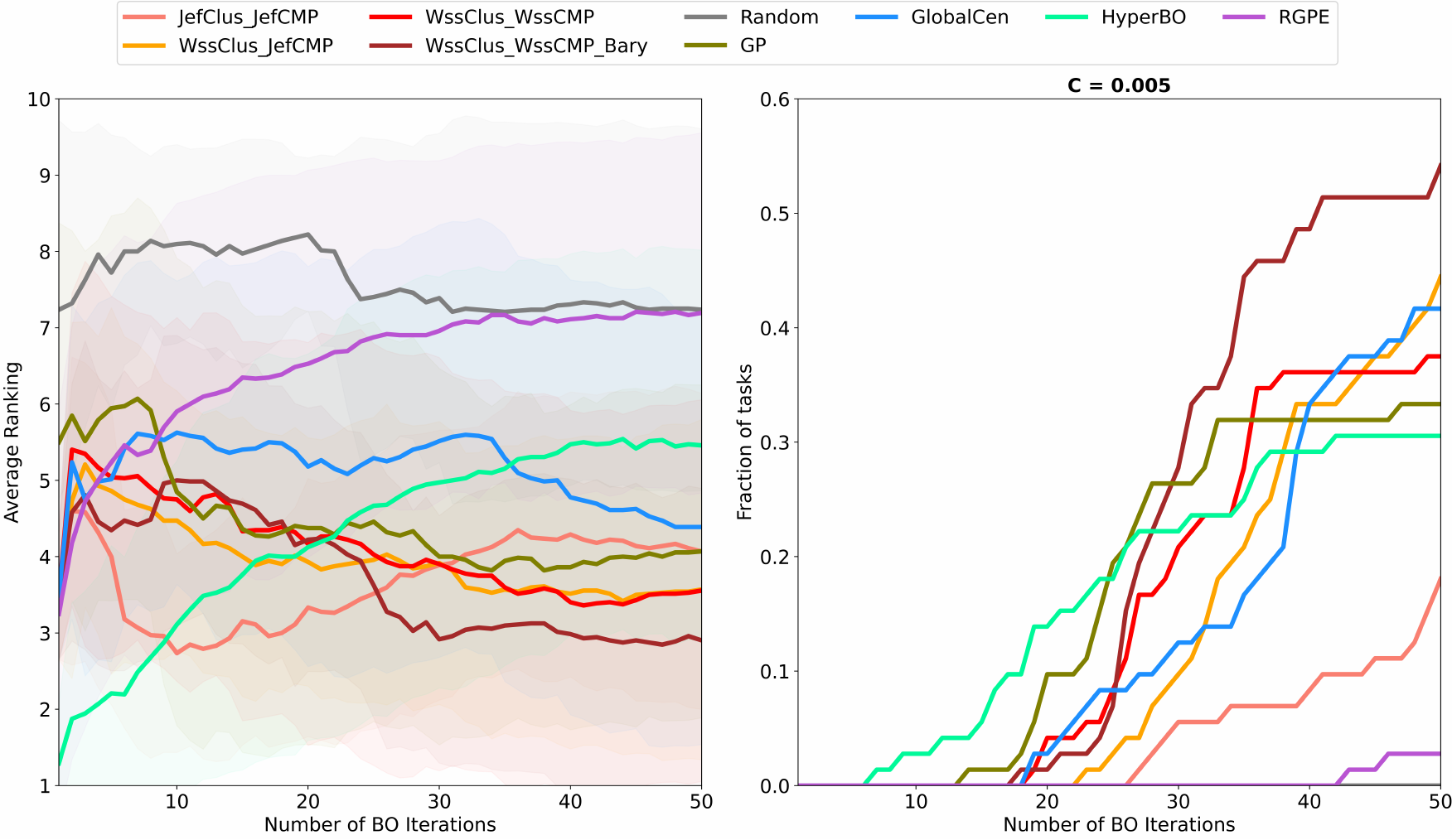}%
\end{minipage}
\end{figure}

\section{Conclusions}\label{sec:Conclusions}

We propose a clustering-based meta-BO framework that leverages historical
source tasks in real-world scenarios without assuming meta-task homogeneity.
To achieve this, we utilize geometric properties derived from statistical
distances throughout all stages, including clustering, cluster prototype
construction, and online adaptive prior construction. Experimental
results and theoretical guarantees demonstrate the robustness of $\model$
in practical meta-BO settings. Future directions include theoretical
analysis when using Wasserstein barycenter, tackling more challenging
scenarios such as hyperparameter optimization for modern deep learning
architectures, and employing more efficient clustering techniques
and promising statistical metric spaces. 

\vspace{-0.35cm}

\bibliographystyle{splncs04}
\bibliography{refs}

\appendix

\newpage

\section{Appendix 1: Towards theoretical guarantees on $\protect\model$}\label{sec:First-appendix}

\subsection{Existence of Adaptive Meta-Prior as Non-Degenerate Gaussian Process}

In this section, we present related theorems and proofs that provide
a theoretical guarantee for the validity of our proposed adaptive
meta-prior, as stated in Theorem \ref{them:gd_linear} and Proposition
\ref{prop:prior_non_degenerate} in Section \ref{subsec:cm-BO}.
\begin{theorem}
\label{them:gp_gd_them}A Gaussian process (GP) $\{X_{t}\}_{t\in T}$
indexed by a set $T$ is a family of (real-valued) random variables
$X_{t}$, all defined on the same probability space, such that for
any finite subset $F\subset T$ the random vector $X_{F}:=\{X_{t}\}_{t\in F}$
has a (possibly degenerate) Gaussian distribution (GD). Then, if these
finite-dimensional distributions are all non-degenerate then the Gaussian
process is said to be non-degenerate.
\end{theorem}

\begin{lemma}
A symmetric matrix $M$ with real entries is \emph{positive-definite}
if the real number $z^{T}Mz$ is positive for every non-zero real
column vector $z$.
\end{lemma}

\begin{theorem}
\label{them:gd_linear_aux}The linear combination of independent Gaussian
random variables is also Gaussian. Suppose that $\mathbf{X}_{1},\mathbf{X}_{2},\ldots,\mathbf{X}_{N}$
are $N$ independent Gaussian distributed random variables, $\mathbf{Y}=\sum_{i=1}^{n}a_{i}\mathbf{X}_{i}$
($a_{i}\in\mathbb{R}$) is a linear combination of $\mathbf{X}_{i},i=1,2,\ldots,N$.
The random variable $\mathbf{Y}$ is also Gaussian, and if $\mathbf{X}_{i}\sim\mathcal{N}(\mu_{i},\sigma_{i}^{2}),i=1,2,\ldots,N$
then $\mathbf{Y}\sim\mathcal{N}(\sum_{i=1}^{N}a_{i}\mu_{i},\sum_{i=1}^{N}a_{i}^{2}\sigma_{i}^{2})$.
\end{theorem}

\begin{proof}
First, it is important to note that a Gaussian distribution is fully
specified by its mean vector and covariance matrix. Therefore, we
can straightforwardly prove that

\begin{equation}
\begin{aligned}
\mathbb{E}[\mathbf{Y}] &= \sum_{i=1}^{N}a_{i}\mu_{i},\\
\mathbb{V}[\mathbf{Y}] &= \sum_{i=1}^{N}a^{2}_{i}\sigma_{i}^{2}\\
\end{aligned}
\end{equation}For the mean, we have

\begin{equation}
\mathbb{E}[\mathbf{Y}] = \mathbb{E}\biggl[\sum_{i=1}^{N}a_{i}\mathbf{X}_{i}\biggr] = \sum_{i=1}^{N}a_{i}\mathbb{E}[\mathbf{\mathbf{X}_i}] = \sum_{i=1}^{N}a_{i}\mu_{i}
\end{equation}from linearity of expectations. 

For the covariance matrix, we can expand $\mathbb{V}[\mathbf{Y}]$
as a result from general Bienaymé's identity:

\begin{equation}
\begin{aligned}
\mathbb{V}[\mathbf{Y}]&=\mathbb{V}\biggl[\sum_{i=1}^{N}a_{i}\mathbf{X}_{i}\biggr]\\
&=\sum_{i,j=1}^{N}a_{i}a_{j}Cov(\mathbf{X}_{i},\mathbf{X}_{j})\\
&=\sum_{i=1}^{N}a_{i}^{2}\mathbb{V}(\mathbf{X}_{i})+2\sum_{1 \le i\le j \le N}a_{i}a_{j}Cov(\mathbf{X}_{i},\mathbf{X}_{j})=\sum_{i=1}^{N}a_{i}^{2}\mathbb{V}(\mathbf{X}_{i})=\sum_{i=1}^{N}a_{i}^{2}\sigma_{i}^{2}
\end{aligned}
\end{equation} noting that $Cov(\mathbf{X}_{i},\mathbf{X}_{j})=0,\forall(i\neq j)$,
which arises from assuming that $\mathbf{X}_{i}$ $(i=1,2,\ldots,N)$
are independent Gaussian random variables.
\end{proof}

\begin{proposition}
\label{prop:pos_def_gd}The multivariate normal distribution (a.k.a
finite-dimensional GD) is said to be non-degenerate when the symmetric
covariance matrix $\mathbf{\Sigma}$ is positive-definite.
\end{proposition}

\begin{lemma}
\label{lem:pos_def_mat_linear}Suppose that $\mathbf{A}_{1},\mathbf{A}_{2},\ldots,\mathbf{A}_{N}$
are $N$ positive definite matrices, then $\sum_{i=1}^{N}a_{i}\mathbf{A}_{i}$
$(a_{i}\in\mathbb{R},a_{i}>0)$ is a positive definite matrix.
\end{lemma}

\begin{proof}
Let ${A}_{i},i=\overline{1,N}$ be positive definite matrices, that
means $\forall h\in\mathbb{R}^{n}$, we have $h^{T}A_{i}h>0,i=\overline{1,N}$.
With $a_{i}\in\mathbb{R},a_{i}>0,i=\overline{1,N}$, it can be directly
inferred that $a_{i}h^{T}A_{i}h>0\hspace{0.1cm}\forall h\in\mathbb{R}^{n},i=\overline{1,N}$.
Hence, 

\begin{equation}
\sum_{i=1}^{N}a_{i}h^{T}A_{i}h > 0\hspace{0.1cm}\forall h \in \mathbb{R}^{n}
\end{equation}From the distributive laws of matrix multiplication, we have

\begin{equation}
\sum_{i=1}^{N}a_{i}h^{T}A_{i}h = h^{T}\left(\sum_{i=1}^{N}a_{i}A_{i}\right)h > 0\hspace{0.1cm}\forall h \in \mathbb{R}^{n}
\end{equation}This implies that $\sum_{i=1}^{N}a_{i}\mathbf{A}_{i}$ $(a_{i}\in\mathbb{R},a_{i}>0)$
is a positive definite matrix.
\end{proof}

\begin{corollary}
\label{cor:gp_gp}Let $X_{k}=\{X_{k,t}\}_{t\in T}$ be a Gaussian
process for $k=\overline{1,N}$. Assuming that any two GPs $\{X_{i,t}\}_{t\in T}$
and $\{X_{j,t}\}_{t\in T}$ ($i\neq j;i,j=\overline{1,N}$) are independent,
we can then conclude that the linear combination of these $N$ GPs,
denoted as $\bar{X}=\left\{ \sum_{k=1}^{N}a_{k}X_{k,t}\right\} _{t\in T}$,
where $a_{i}\in\mathbb{R}$, $a_{i}>0$, and $i=\overline{1,N}$,
is also a GP.
\end{corollary}

\begin{proof}
For any finite index subset $F\subset T$, 

the random vector $\bar{X}_{F}:=\{\sum_{k=1}^{N}a_{k}X_{k,t}\}_{t\in F}$
derived from $\left\{ \sum_{k=1}^{N}a_{k}X_{k,t}\right\} _{t\in T}$
follows a GD. This is because each $\{X_{k,t}\}_{t\in F}$, where
$k=\overline{1,N}$, also has a GD due to the fact that each $X_{k}=\{X_{k,t}\}_{t\in T}$
is a GD (directly inferred from Theorem \ref{them:gp_gd_them} and
Theorem \ref{them:gd_linear_aux}). Therefore, we can conclude that
$\bar{X}=\left\{ \sum_{k=1}^{N}a_{k}X_{k,t}\right\} _{t\in T}$ is
a GP over the index set $T$.

It is straightforward to derive the result of our Theorem \ref{them:gd_linear}
from Corollary \ref{cor:gp_gp}.
\end{proof}

\begin{corollary}
\label{cor:non_deg_gp}Suppose that all GPs $X_{k}=\{X_{k,t}\}_{t\in T},k=\overline{1,N}$
(from \ref{cor:gp_gp}) are non-degenerate. Then, $\bar{X}=\left\{ \sum_{k=1}^{N}a_{k}X_{k,t}\right\} _{t\in T}$
is also a non-degenerate GP.
\end{corollary}

\textit{Proof.} Using the notation from Corollary \ref{cor:gp_gp},
where $X_{k}=\{X_{k,t}\}_{t\in T}$ is a non-degenerate GP, we can
conclude that every random vector $\{X_{k,t}\}_{t\in F}$ derived
from it over an index subset $F$ follows a non-degenerate GD (Theorem
\ref{them:gp_gd_them}). Therefore, the covariance matrix $\Sigma_{k;t\in F}$
of the corresponding GD for $\{X_{k,t}\}_{t\in F}$ is positive-definite
(Proposition \ref{prop:pos_def_gd}). Additionally, we have known
that the random vector $\bar{X}_{F}:=\{\sum_{k=1}^{N}a_{k}X_{k,t}\}_{t\in F}$
derived from $\left\{ \sum_{k=1}^{N}a_{k}X_{k,t}\right\} _{t\in T}$
follows a GD. We have that this distribution has its covariance matrix
given by $\sum_{k}^{N}a_{k}\Sigma_{k;t\in F}$, which is also positive-definite
(according to Theorem \ref{them:gd_linear_aux} and Lemma \ref{lem:pos_def_mat_linear}).
This implies that $\bar{X}_{F}$ also follows a non-degenerate GD
for every $F\subset T$ (as defined in Proposition \ref{prop:pos_def_gd}).
Therefore, we can conclude that $\left\{ \sum_{k=1}^{N}a_{k}X_{k,t}\right\} _{t\in T}$
is a non-degenerate GP.
\begin{proposition}
\label{prop:pos_def_kernel}A kernel is a symmetric continuous function:
$K:[a,b]\times[a,b]\rightarrow\mathbb{R}$ where symmetric means that
$K(x,y)=K(y,x)$ for all $x,y\in[a,b]$. $K$ is said to be a positive-definite
kernel if and only if
\begin{equation}
\sum_{i=1}^{n}\sum_{j=1}^{n}K(x_{i},x_{j})c_{i}c_{j}\ge0
\end{equation}
for all finite sequences of points $x_{1},\ldots,x_{n}$ of $[a,b]$
and all choices of real numbers $c_{1},\ldots,c_{n}$. This is also
a special case of Mercer's theorem. The equation holds when $c_{i}=0$
$(\forall i)$.
\end{proposition}

\begin{proposition}
\label{prop:pos_def_char}A $n\times n$ symmetric real matrix $M$
is said to be \textbf{positive-definite} if $\mathbf{x}^{T}M\mathbf{x}>0$
for all non-zero $\mathbf{x}$ in $\mathbb{R}^{n}$. Formally, 
\begin{align}
M\text{ positive-definite \ensuremath{\Longleftrightarrow}} & \forall x\in\mathbb{R}^{n}\backslash\{\mathbf{0}\},\mathbf{x}^{T}M\mathbf{x}>0\\
\Longleftrightarrow & \forall x_{1},\ldots,x_{n}\in\mathbb{R}\hspace{0.1cm}|\hspace{0.1cm}\exists x_{k},x_{k}\neq0\\
 & \sum_{i}^{n}\sum_{j}^{n}M_{ij}x_{i}x_{j}>0
\end{align}
\end{proposition}

\begin{lemma}
\label{lem:eigen}Let $M$ be an $n\times n$ Hermitian matrix (this
includes real symmetric matrices). $M$ is positive-definite if and
only if all of its eigenvalues are positive.
\end{lemma}

\begin{corollary}
\label{cor:kernel_pos_def}A kernel function is \textbf{positive-definite}
if and only if, for any finite set of input points, the corresponding
covariance matrix is \textbf{positive-definite }(directly deduced
from Proposition \ref{prop:pos_def_kernel} and Proposition \ref{prop:pos_def_char}).
\end{corollary}

\begin{theorem}
\label{thm:pos_def_post}Consider a Gaussian process, denoted as $f\sim\mathcal{GP}(\mu,k)$,
where $k$ is a positive-definite kernel function. When conditioning
its prior on training inputs $\mathbf{X}$ to obtain a posterior Gaussian
process, any posterior predictive distribution made on test inputs
$\mathbf{X}_{*}$ follows a Gaussian distribution $\mathcal{N}\sim(\mathbf{m}_{*},\mathbf{K}_{*})$,
where $\mathbf{K}_{*}$ is a positive-definite matrix.
\end{theorem}

\begin{proof}
W.l.o.g, assume that $\mu$ is a zero-mean function, which implies
that $f\sim\mathcal{GP}(0,k)$. Consider the case where observations
at training inputs $\mathbf{X}$ are \textit{noise-free} (infer similarly
for the \emph{noise} case), the joint distribution of the training
outputs $f$ and the test outputs $f_{*}$ according to the prior
can be expressed as follows: 
\begin{equation}
\begin{bmatrix}f_{*}\\
f
\end{bmatrix}\sim\mathcal{N}\left(\mathbf{0},\begin{bmatrix}K(\mathbf{X_{*}},\mathbf{X_{*}}) & K(\mathbf{X_{*}},\mathbf{X})\\
K(\mathbf{X},\mathbf{X_{*}}) & K(\mathbf{X},\mathbf{X})
\end{bmatrix}\right)
\end{equation}
Let there be $n$ training points and $n_{*}$ test points. The $n\times n_{*}$
matrix $K(\mathbf{X},\mathbf{X_{*}})$ represents the covariances
evaluated at all pairs of training and test points. This is similar
for $K(\mathbf{X},\mathbf{X}),K(\mathbf{X_{*}},\mathbf{X})$ and $K(\mathbf{X_{*}},\mathbf{X_{*}})$.
By conditioning the joint Gaussian prior distribution on the observations
$\mathbf{X}$, we obtain the posterior predictive distribution $f_{*}|\mathbf{X}_{*},\mathbf{X},f\sim\mathcal{N}(K(\mathbf{X_{*}},\mathbf{X})K(\mathbf{X},\mathbf{X})^{-1}f,K(\mathbf{X_{*}},\mathbf{X_{*}})-K(\mathbf{X_{*}},\mathbf{X})K(\mathbf{X},\mathbf{X})^{-1}K(\mathbf{X},\mathbf{X_{*}}))$.
Under the assumption that our GP employs a positive-definite kernel
function, it can be deduced that the covariance matrix of the prior
joint distribution, $\mathbf{M}:=\begin{bmatrix}K(\mathbf{X_{*}},\mathbf{X_{*}}) & K(\mathbf{X_{*}},\mathbf{X})\\
K(\mathbf{X},\mathbf{X_{*}}) & K(\mathbf{X},\mathbf{X})
\end{bmatrix}$, is positive-definite (as stated in Corollary \ref{cor:kernel_pos_def}).
Moreover, by using the \textit{Schur complement}, this matrix can
be decomposed into $\mathbf{M}=\mathbf{D}\mathbf{P}\mathbf{D}^{T}$where
$\mathbf{D}:=\begin{bmatrix}I & K(\mathbf{X}_{*},\mathbf{X})K(\mathbf{X},\mathbf{X})^{-1}\\
0 & I
\end{bmatrix}$ and $\mathbf{P}:=\begin{bmatrix}K(\mathbf{X_{*}},\mathbf{X_{*}})-K(\mathbf{X_{*}},\mathbf{X})K(\mathbf{X},\mathbf{X})^{-1}K(\mathbf{X},\mathbf{X_{*}}) & 0\\
0 & K(\mathbf{X},\mathbf{X})
\end{bmatrix}$.

We can separate the set of eigenvalues of matrix $\mathbf{M}$ into
two subsets, namely the set of eigenvalues of $K(\mathbf{X_{*}},\mathbf{X_{*}})-K(\mathbf{X_{*}},\mathbf{X})K(\mathbf{X},\mathbf{X})^{-1}K(\mathbf{X},\mathbf{X_{*}})$
and $K(\mathbf{X},\mathbf{X})$. Since all eigenvalues of $\mathbf{M}$
are positive (Lemma \ref{lem:eigen}), all eigenvalues of $K(\mathbf{X_{*}},\mathbf{X_{*}})-K(\mathbf{X_{*}},\mathbf{X})K(\mathbf{X},\mathbf{X})^{-1}K(\mathbf{X},\mathbf{X_{*}})$
are also positive. It is also evident that $K(\mathbf{X_{*}},\mathbf{X_{*}})-K(\mathbf{X_{*}},\mathbf{X})K(\mathbf{X},\mathbf{X})^{-1}K(\mathbf{X},\mathbf{X_{*}})$
as each term involved is a symmetric matrix. Therefore, it is positive-definite.
\end{proof}

\begin{corollary}
\label{cor:pos_non_deg}Given a prior GP with a positive-definite
kernel function, it follows that the posterior GP, conditioned on
the training inputs $\mathcal{D}$, is non-degenerate (deduced from
the results of Theorem \ref{thm:pos_def_post} and \ref{them:gp_gd_them}). 
\end{corollary}

According to Bochner's theorem, a continuous stationary function $k(x,y)=\tilde{k}(|x-y|)$
is positive-definite if and only if $\tilde{k}$ is the Fourier transform
of a finite positive measure: $\tilde{k}(t)=\int_{\mathbb{R}}e^{-i\omega t}d\mu(\omega)$.
The Matérn kernels can be expressed as the Fourier transforms of the
function $\frac{1}{(1+\omega^{2})^{p}}$; thus, these kernels exhibit
positive-definiteness. From here, we can directly conclude Proposition
\ref{prop:prior_non_degenerate} in Section \ref{subsec:cm-BO} with
Corollary \ref{cor:pos_non_deg}.

\subsection{Meta-Task Effects on Regret Bounds and the Facilitating Role of $\protect\model$}\label{subsec:Meta-Task-Effects-on}

We analyze the derived upper bound on\emph{ }regret during BO procedures,
as stated in Theorem \ref{them:regret}. We then provide a theoretical
explanation of how $\model$ can accelerate the convergence rate of
BO by leveraging meta-tasks that are progressively similar to the
true target function, as well as its robustness in handling heterogeneity
in function structures across meta-tasks. For all notation used in
this section, we recommend that readers refer to Section \ref{subsec:cm-BO}.
\begin{lemma}
\label{lem:true_posterior_gap}\cite{dai2022provably} Let $\delta\in(0,1)$.
Define $\gamma_{\tau}$ as the maximum information gain regarding
the target function $f$ that can be obtained by observing any set
of $\tau$ observations. $\mu_{\tau}(x)$ and $k_{\tau}(x)$ are the
derived posterior mean and covariance function for iteration $\tau$,
respectively. If $\beta_{\tau}=B+\sigma\sqrt{2(\gamma_{\tau-1}+1+log(\frac{4}{\delta}))}$,
with a probability of at least $1-\frac{\delta}{4}$, $|f(x)-\mu_{\tau}(x)|\le\beta_{\tau}\sqrt{k_{\tau}(x)},\forall x\in\mathcal{D},\tau\ge1$.
\end{lemma}

We already have $\mu_{0}^{(\tau)}:=\sum_{i=1}^{C}w_{\mathcal{C}_{i}}^{(\tau-1)}\mu^{\mathcal{C}_{i}}$,
and $k_{0}^{(\tau)}:=\sum_{i=1}^{C}\left(w_{\mathcal{C}_{i}}^{(\tau-1)}\right)^{2}k^{\mathcal{C}_{i}}$
being the mean function and covariance function associated with our
meta-prior for iteration $\tau$. Completely analogously, we also
have $\tilde{\mu}_{0}^{(\tau)}$ and $\tilde{k}_{0}^{(\tau)}$ when
$K$ meta-tasks are replaced by $\tilde{\mathcal{D}}^{t_{i}}$. Furthermore,
for cluster representatives $\mathcal{GP}\left(\mu^{\mathcal{C}_{i}},k^{\mathcal{C}_{i}}\right)$
as the \emph{geometric center, }we can expand $\mu_{0}^{(\tau)}(x)=\sum_{i=1}^{K}\zeta_{i}^{\tau}\mu^{t_{i}}(x)$,
$k_{0}^{(\tau)}(x)=\sum_{i=1}^{K}\zeta^{'}{}_{i}^{\tau}k^{t_{i}}(x)$;
$\tilde{\mu}_{0}^{(\tau)}(x)=\sum_{i=1}^{K}\zeta_{i}^{\tau}\tilde{\mu}^{t_{i}}(x)$,
$\tilde{k}_{0}^{(\tau)}(x)=\sum_{i=1}^{K}\zeta^{'}{}_{i}^{\tau}\tilde{k}^{t_{i}}(x)$
where $\zeta_{i}^{\tau}:=\frac{w_{\mathcal{C}_{i}}^{(\tau-1)}}{|C_{i}|}$,
$\zeta^{'}{}_{i}^{\tau}:=\frac{\left(w_{\mathcal{C}_{i}}^{(\tau-1)}\right)^{2}}{|C_{i}|}$
represent the weights of each meta-task derived posterior model, which
are linked to the weights of the cluster that it belongs to. Because
$\sum_{i=1}^{C}w_{\mathcal{C}_{i}}^{(\tau-1)}=1$, we easily have
$\sum_{i=1}^{K}\zeta_{i}^{\tau}=1$ and $\sum_{i=1}^{K}\zeta^{'}{}_{i}^{\tau}<1$.
\begin{lemma}
\label{lem:mean_meta_true_gap}With a probability of at least $1-\frac{\delta}{4}$,
$\Biggl|\mu_{0}^{(\tau)}(x)-\tilde{\mu}_{0}^{(\tau)}(x)\Biggr|\le\alpha_{\tau},\forall x\in\mathcal{D}$ 

where $\alpha_{\tau}:=\sum_{i=1}^{K}\zeta_{i}^{\tau}\frac{n_{i}}{\sigma^{2}}\left(D_{i}+2\sqrt{2\sigma^{2}log\frac{8n_{i}}{\delta}}\right)$.
\end{lemma}

\begin{proof}
This lemma can be proceeded by expanding $\mu_{0}^{(\tau)}(x)$, $\tilde{\mu}_{0}^{(\tau)}(x)$
and then transforming it in the same way as proof of Lemma 3 in \cite{dai2022provably},
noting that: 
\end{proof}

\begin{enumerate}
\item[(1)] $\left\Vert (\mathbf{\Sigma}+\sigma^{2}I)^{-1}\right\Vert _{2}\le\frac{1}{\sigma^{2}}$
where $\mathbf{\Sigma}$ is a positive semi-definite (Lemma 5 \cite{dai2022provably}), 
\item[(2)] $\|\mathbf{y}^{t_{i}}-\mathbf{f}^{t_{i}}\|_{2}\le C$, $\|\tilde{\mathbf{y}}^{t_{i}}-\tilde{\mathbf{f}}^{t_{i}}\|_{2}\le C$
where $C:=\sqrt{n_{i}}\sqrt{2\sigma^{2}log\frac{8n_{i}}{\delta}}$
with a probability at least $1-\frac{\delta}{4}$ (Lemma 6 \cite{dai2022provably}), 
\item[(3)] $k_{0}(x)=\tilde{k}_{0}(x)$ for all $x\in\mathcal{D}$, as the posterior
covariance depends solely on the input locations, independent of the
output responses,
\item[(4)] the assumption w.l.o.g. that $k(x,x')\le1$ for all $x,x'\in\mathcal{D}$.
\end{enumerate}
We now proceed to prove our Theorem \ref{them:regret}. 

\textbf{\emph{Proof for Theorem}} \ref{them:regret}. Denote $a^{(\tau)}(x)=\mu^{(\tau)}(x)+\xi\sqrt{k^{(\tau)}(x)}\,\,(\xi>0)$
as the GP-UCB constructed from the corresponding posterior mean and
covariance. Let $x^{*}$ be a global maximizer of the target function
$f$, and $x_{\tau}$ the $\tau$-th queried observation. We can expand
the\emph{ instantaneous regret }as follows 
\begin{align}
 & r_{\tau}=f(x^{*})-f(x_{\tau})\nonumber \\
 & \le\mu^{(\tau)}(x^{*})+\beta_{\tau}\sqrt{k^{(\tau)}(x^{*})}-\mu^{(\tau)}(x_{\tau})+\beta_{\tau}\sqrt{k^{(\tau)}(x_{\tau})}\nonumber \\
 & =\mu^{(\tau)}(x^{*})+\frac{\beta_{\tau}}{\xi}\left(a^{(\tau)}(x^{*})-\mu^{(\tau)}(x^{*})\right)-\mu^{(\tau)}(x_{\tau})+\beta_{\tau}\sqrt{k^{(\tau)}(x_{\tau})}\nonumber \\
 & \le\mu^{(\tau)}(x^{*})+\frac{\beta_{\tau}}{\xi}\left(a^{(\tau)}(x_{\tau})-\mu^{(\tau)}(x^{*})\right)-\mu^{(\tau)}(x_{\tau})+\beta_{\tau}\sqrt{k^{(\tau)}(x_{\tau})}\nonumber \\
 & =\left(1-\frac{\beta_{\tau}}{\xi}\right)\left(\mu^{(\tau)}(x^{*})-\mu^{(\tau)}(x_{\tau})\right)+(\xi+\beta_{\tau})\sqrt{k^{(\tau)}(x_{\tau})}\label{eq:regret_1}
\end{align}
The first inequality follows Lemma \ref{lem:true_posterior_gap} for
all $x\in\mathcal{D}$; the second one is because $x_{\tau}$ is selected
to to maximize the acquisition function. 

Additionally, we have 
\[
\mu^{(\tau)}(x^{*})-\mu^{(\tau)}(x_{\tau})=\left(\mu_{0}^{(\tau)}(x^{*})-\mu_{0}^{(\tau)}(x_{\tau})\right)+\left(\mathbf{k}_{0}^{(\tau)}(x^{*})-\mathbf{k}_{0}^{(\tau)}(x_{\tau})\right)^{T}(\mathbf{\Sigma}_{0}^{(\tau)}+\sigma^{2}\mathbf{I})^{-1}(\mathbf{y_{\tau}}-\mu_{0}^{(\tau)}(\mathbf{X_{\tau}}))
\]
in which the first term can be transformed
\begin{align}
\mu_{0}^{(\tau)}(x^{*})-\mu_{0}^{(\tau)}(x_{\tau}) & =\left(\mu_{0}^{(\tau)}(x^{*})-\tilde{\mu}_{0}^{(\tau)}(x^{*})\right)+\left(\tilde{\mu}_{0}^{(\tau)}(x_{\tau})-\mu_{0}^{(\tau)}(x_{\tau})\right)\nonumber \\
 & +\left(\tilde{\mu}_{0}^{(\tau)}(x^{*})-f(x^{*})\right)+\left(f(x_{\tau})-\tilde{\mu}_{0}^{(\tau)}(x_{\tau})\right)+r_{\tau}\nonumber \\
 & \le2\alpha+\left(\tilde{\mu}_{0}^{(\tau)}(x^{*})-f(x^{*})\right)\nonumber \\
 & +\left(f(x_{\tau})-\tilde{\mu}_{0}^{(\tau)}(x_{\tau})\right)+r_{\tau}\\
 & =2\alpha+\eta^{\tau}(x^{*})-\eta^{\tau}(x_{\tau})+r_{\tau}\label{eq:regret_3}
\end{align}
where $\eta^{\tau}(x):=\tilde{\mu}_{0}^{(\tau)}(x)-f(x)$, as described
in Theorem \ref{them:regret}. The inequality follows Lemma \ref{lem:mean_meta_true_gap}.

Meanwhile, 
\begin{align*}
 & \left|\left(\mathbf{k}_{0}^{(\tau)}(x^{*})-\mathbf{k}_{0}^{(\tau)}(x_{\tau})\right)^{T}(\mathbf{\Sigma}_{0}^{(\tau)}+\sigma^{2}\mathbf{I})^{-1}(\mathbf{y_{\tau}}-\mu_{0}^{(\tau)}(\mathbf{X_{\tau}}))\right|\\
 & \le\|\mathbf{k}_{0}^{(\tau)}(x^{*})-\mathbf{k}_{0}^{(\tau)}(x_{\tau})\|_{2}\|(\mathbf{\Sigma}_{0}^{(\tau)}+\sigma^{2}\mathbf{I})^{-1}\|_{2}\|\mu_{0}^{(\tau)}(\mathbf{X_{\tau}})-\mathbf{y_{\tau}}\|_{2}
\end{align*}
in which 
\begin{align*}
\|\mathbf{k}_{0}^{(\tau)}(x^{*})-\mathbf{k}_{0}^{(\tau)}(x_{\tau})\|_{2} & =\left\Vert \sum_{i=1}^{K}\zeta^{'}{}_{i}^{\tau}\left(\mathbf{k}^{t_{i}}(x^{*})-\mathbf{k}^{t_{i}}(x_{\tau})\right)\right\Vert _{2}\\
 & \le\sum_{i=1}^{K}\zeta^{'}{}_{i}^{\tau}\|\mathbf{k}^{t_{i}}(x^{*})-\mathbf{k}^{t_{i}}(x_{\tau})\|_{2}\le2\sum_{i=1}^{K}\zeta^{'}{}_{i}^{\tau}\sqrt{n_{i}}
\end{align*}
with the inequalities follow the triangle inequality of norm and the
assumption $k(x,x')\le1$ for all $x,x'\in\mathcal{D},$ respectively;
and,
\begin{align}
\|\mu_{0}^{(\tau)}(\mathbf{X_{\tau}})-\mathbf{y_{\tau}}\|_{2} & =\left\Vert \tilde{\mu}_{0}^{(\tau)}(\mathbf{X_{\tau}})-\mathbf{f_{\tau}}+\mu_{0}^{(\tau)}(\mathbf{X_{\tau}})-\tilde{\mu}_{0}^{(\tau)}(\mathbf{X_{\tau}})\right\Vert _{2}\nonumber \\
 & \le\left\Vert \eta^{\tau}(\mathbf{X_{\tau}})\right\Vert _{2}+\left\Vert \mu_{0}^{(\tau)}(\mathbf{X_{\tau}})-\tilde{\mu}_{0}^{(\tau)}(\mathbf{X_{\tau}})\right\Vert _{2}\nonumber \\
 & \le\left\Vert \eta^{\tau}(\mathbf{X_{\tau}})\right\Vert _{2}+\sqrt{\sum_{i=1}^{\tau-1}\left(\mu_{0}^{(\tau)}(x_{i})-\tilde{\mu}_{0}^{(\tau)}(x_{i})\right)^{2}}\nonumber \\
 & \le\left\Vert \eta^{\tau}(\mathbf{X_{\tau}})\right\Vert _{2}+\sqrt{\tau-1}\alpha_{\tau}\label{eq:eq_2}
\end{align}
with the first inequality follows the triangle inequality and the
second one follows Lemma \ref{lem:mean_meta_true_gap}.

Combining all of these above expansions with a note from Lemma 5 \cite{dai2022provably},
we have

\begin{align*}
 & \left|\left(\mathbf{k}_{0}^{(\tau)}(x^{*})-\mathbf{k}_{0}^{(\tau)}(x_{\tau})\right)^{T}(\mathbf{\Sigma}_{0}^{(\tau)}+\sigma^{2}\mathbf{I})^{-1}(\mathbf{y_{\tau}}-\mu_{0}^{(\tau)}(\mathbf{X_{\tau}}))\right|\\
 & \le2\sum_{i=1}^{K}\zeta^{'}{}_{i}^{\tau}\frac{\sqrt{n_{i}}}{\sigma^{2}}\left(\left\Vert \eta^{\tau}(\mathbf{X_{\tau}})\right\Vert _{2}+\sqrt{\tau-1}\alpha_{\tau}\right)
\end{align*}

Finally, we can adopt

$r_{\tau}\le\left(1-\frac{\beta_{\tau}}{\xi}\right)\left(\iota_{\tau}\alpha_{\tau}+\eta^{\tau}(x^{*})-\eta^{\tau}(x_{\tau})+r_{\tau}+\omega_{\tau}\left\Vert \eta^{\tau}(\mathbf{X_{\tau}})\right\Vert _{2}\right)+(\xi+\beta_{\tau})\sqrt{k^{(\tau)}(x_{\tau})}$
in which $\iota_{\tau}:=2+\omega_{\tau}\sqrt{\tau-1}$ and $\omega_{\tau}:=2\sum_{i=1}^{K}\zeta^{'}{}_{i}^{\tau}\frac{\sqrt{n_{i}}}{\sigma^{2}}$. 

This is equivalent to 

\begin{align*}
r_{\tau} & \le\frac{\xi}{\beta_{\tau}}\left(1-\frac{\beta_{\tau}}{\xi}\right)\left(\iota_{\tau}\alpha_{\tau}+\eta^{\tau}(x^{*})-\eta^{\tau}(x_{\tau})+\omega_{\tau}\left\Vert \eta^{\tau}(\mathbf{X_{\tau}})\right\Vert _{2}\right)+\frac{\xi}{\beta_{\tau}}(\xi+\beta_{\tau})\sqrt{k^{(\tau)}(x_{\tau})}\\
 & =\mathbf{A}_{1}+\mathbf{A}_{2}
\end{align*}

From this, we can conclude our Theorem \ref{them:regret}.

\medskip{}

\textbf{\textit{More insights into Remark \ref{rem:regret}}}. The
impact of the meta-task on BO queries can be interpreted through $\mathbf{A}_{1}$
and $\mathbf{A}_{2}$. The term $\mathbf{A}_{2}$ is proportional
to the upper bound on the instantaneous regret for the standard GP-UCB
\cite{srinivas2009gaussian}. Therefore, the meta-tasks affect the
upper bound on the regret through the term $\mathbf{A}_{1}$, which
is of concern when it is positive. Specifically, two related factors
are explored:
\begin{enumerate}
\item[(1)] $\iota_{\tau}\alpha_{\tau}$ where $\alpha_{\tau}:=\sum_{i=1}^{K}\zeta_{i}^{\tau}\frac{n_{i}}{\sigma^{2}}\left(D_{i}+2\sqrt{2\sigma^{2}log\frac{8n_{i}}{\delta}}\right)$,
$\sum_{i=1}^{K}\zeta_{i}^{\tau}=1$. Intuitively, $\alpha_{\tau}$
is larger when the \emph{function gap} between the true functions
of the meta-task and the target task, $D_{i}$, is larger, as well
as the increasing number of meta-points $n_{i}$. Our $\model$ weights
\emph{meta-task clusters} based on $\tau$-iteration's \emph{estimated
$d_{i}^{(\tau)}\approx D_{i}$} using the distance between the function
shape-simulating GDs and the online-shape distribution of the current
task. Accordingly, the weight $\zeta_{i}^{\tau}=\frac{w_{\mathcal{C}_{i}}^{(\tau-1)}}{|C_{i}|}$
reflects $d_{i}^{(\tau)}$ at iteration $\tau$, and allows for adaptive
adjustments to down-weight the contribution of \emph{dissimilar meta-task
clusters}. This reduces $\alpha_{\tau}$, facilitating a tighter bound.
Additionally, meta-tasks within a highly homogeneous cluster will
receive equal weights, which enhances computational efficiency.
\item[(2)] The pattern $\eta^{\tau}(x):=\tilde{\mu}_{0}^{(\tau)}(x)-f(x)$.
From here, we have\\
$\eta^{\tau}(x)=\sum_{i=1}^{K}\zeta_{i}^{\tau}\left(\tilde{\mu}^{t_{i}}(x)-f(x)\right)$,
$\sum_{i=1}^{K}\zeta_{i}^{\tau}=1$. Because $\tilde{\mu}^{t_{i}}(x)$
is the posterior mean conditioning on $n_{i}$ observations of the
target function, it can be inferred that as $n_{i}\rightarrow\infty$,
$\tilde{\mu}^{t_{i}}(x)$ asymptotically approaches the true function
shape of $f(x)$, thus $\eta^{\tau}(x)$ depends mainly on accumulated
observation noises. Our $\model$ weights $\zeta_{i}^{\tau}$ reduce
the influence of \emph{dissimilar meta-tasks} when $\tilde{\mu}^{t_{i}}(x)-f(x)$
is large, helping $\eta^{\tau}(x)$ to approach $0$ faster when observation
noise is not considered and thus facilitating a tighter regret bound.
\end{enumerate}

\section{Appendix 2: On Varying Meta-Task Cluster Cardinality towards cm-BO}\label{sec:Second-appendix}

Fig. \ref{fig:clus_analysis} shows the estimated number of clusters
impacts BO convergence rate. This could be attributed to the homogeneity
of the intra-cluster meta-tasks. However, partitioning meta-tasks
into clusters has a generally positive effect compared to non-clustered
baselines (as 1-C). 
\begin{figure}[H]
\caption{Mean NSR when varying the number of meta-task clusters for the two
$\protect\model$ variants, WssClus\_WssCMP and WssClus\_WssCMP\_Bary
(experimental results from a train-test split seed in \emph{rpart.preproc}
meta-dataset).}\label{fig:clus_analysis}

\includegraphics[width=1\linewidth]{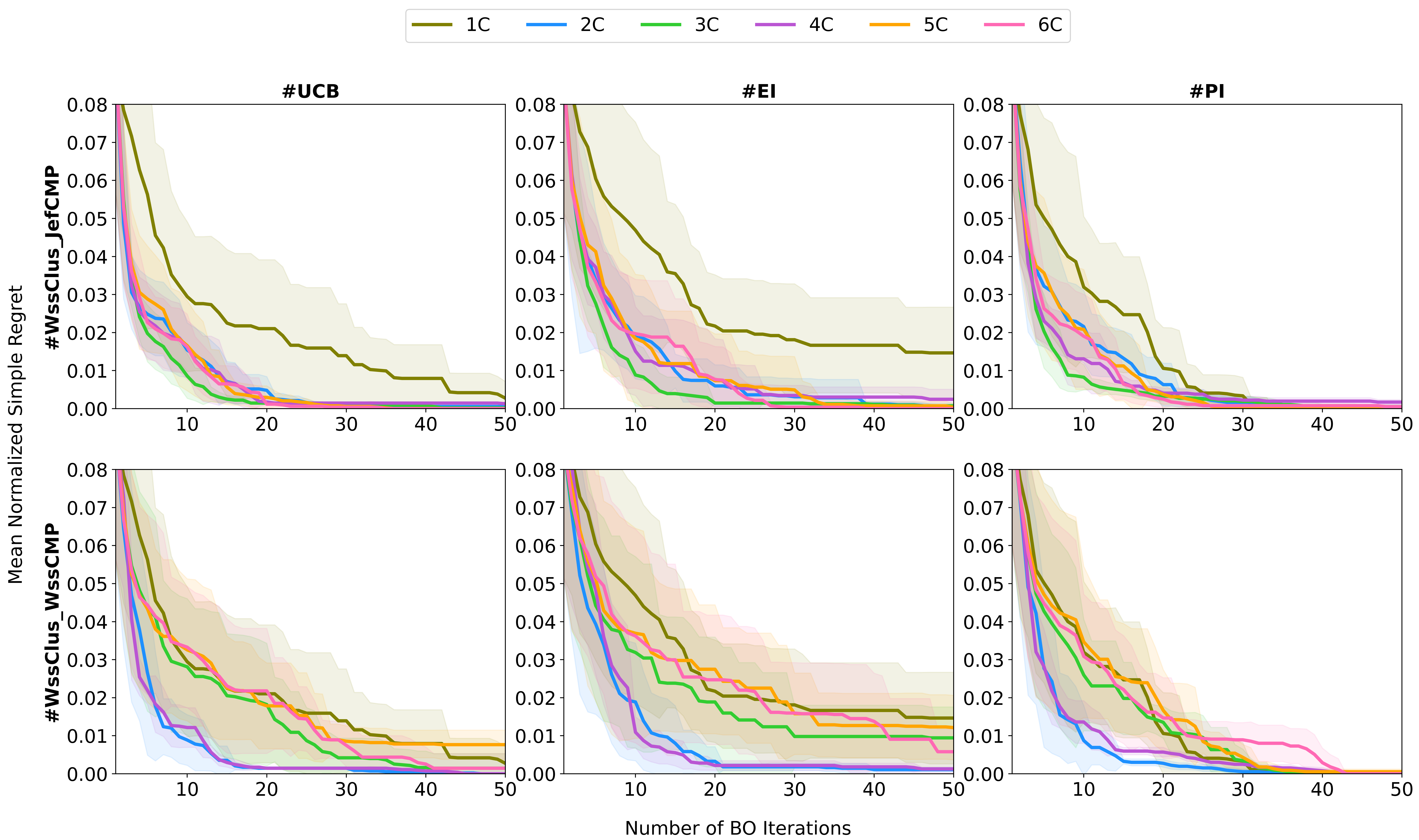}
\end{figure}

\end{document}